\documentclass[letterpaper, 10 pt, journal, twoside]{IEEEtran}

\usepackage{amsmath}
\usepackage{amssymb}
\usepackage[hidelinks]{hyperref}

\usepackage{amsthm}
\usepackage{cite}
\usepackage{algorithm}
\usepackage[noend]{algpseudocode}
\usepackage{color}
\usepackage{graphicx}
\usepackage{booktabs}
\usepackage{xcolor}

\newtheorem{thm}{Theorem}

\newtheorem{lem}{Lemma}

\newtheorem{defn}{Definition}[section]

\newtheorem{conj}{Conjecture}[section]

\newtheorem{rem}{Remark}

\begin{document}
\title{Safe Navigation and Obstacle Avoidance Using \\ Differentiable Optimization Based Control Barrier Functions}

\author{Bolun~Dai$^{1}$, Rooholla~Khorrambakht$^{1}$,~\IEEEmembership{Student Member,~IEEE}, Prashanth~Krishnamurthy$^{1}$,~\IEEEmembership{Member,~IEEE},\\  Vin\'icius~Gon\c{c}alves$^{2}$, Anthony~Tzes$^{2}$,~\IEEEmembership{Senior Member,~IEEE}, Farshad~Khorrami$^{1}$,~\IEEEmembership{Senior Member,~IEEE}%
\thanks{This work was supported by the NYUAD Center for Artificial Intelligence and Robotics (CAIR), funded by Tamkeen under the NYUAD Research Institute Award CG010.}%
\thanks{$^{1}$Bolun Dai, Rooholla Khorrambakht, Prashanth Krishnamurthy, and Farshad Khorrami are with Control/Robotics Research Laboratory, Electrical~\&~Computer Engineering Department, Tandon School of Engineering, New York University, Brooklyn, NY 11201
{\tt\footnotesize \{bd1555, rk4342, prashanth.krishnamurthy, khorrami\}@nyu.edu}}%
\thanks{$^{2} $Vin\'icius Gon\c{c}alves and Anthony Tzes are with Electrical Engineering, New York University Abu Dhabi, Abu Dhabi 129188, United Arab Emirates
{\tt\footnotesize \{vmg6973, anthony.tzes\}@nyu.edu}}
}

\markboth{IEEE Robotics and Automation Letters. Preprint Version. Accepted July, 2023}
{Dai \MakeLowercase{\textit{et al.}}: Safe Navigation and Obstacle Avoidance Using Differentiable Optimization Based Control Barrier Functions}

\maketitle

\begin{abstract}
Control barrier functions (CBFs) have been widely applied to safety-critical robotic applications. However, the construction of control barrier functions for robotic systems remains a challenging task. Recently, collision detection using differentiable optimization has provided a way to compute the minimum uniform scaling factor that results in an intersection between two convex shapes and to also compute the Jacobian of the scaling factor. In this paper, we propose a framework that uses this scaling factor, with an offset, to systematically define a CBF for obstacle avoidance tasks. We provide theoretical analyses of the continuity and continuous differentiability of the proposed CBF. We empirically evaluate the proposed CBF's behavior and show that the resulting optimal control problem is computationally efficient, which makes it applicable for real-time robotic control. We validate our approach, first using a 2D mobile robot example, then on the Franka-Emika Research~3 (FR3) robot manipulator both in simulation and experiment. 
\end{abstract}
\IEEEpeerreviewmaketitle

\section{Introduction}
\IEEEPARstart{S}{afety} is a key consideration when designing control algorithms for robotic applications~\cite{DaiHKK23, DaiKPK23, AmesCENST19} considering rapid integration of robotic systems into our daily lives~\cite{DaiKPK21}. Model predictive control (MPC) and trajectory optimization (TO) based methods have been widely used for safety-critical robot applications, e.g., obstacle avoidance. However, the computation time of MPC and TO based methods limits their deployment on systems requiring fast response time. Additionally, for MPC-based approaches, safety is only guaranteed within the preview horizon. A short preview horizon might lead to abrupt actions to ensure safety, while large preview horizons increase the computation time. A new control paradigm, CBF-based control~\cite{AmesCENST19}, has become popular for safe robotic control since it provides a simple and computationally efficient way for safe control synthesis. Another benefit of CBFs is that CBF constraints take safety into consideration even far away from the safe set boundary.

\begin{figure}[t!]
    \centering
    \includegraphics[width=0.45\textwidth]{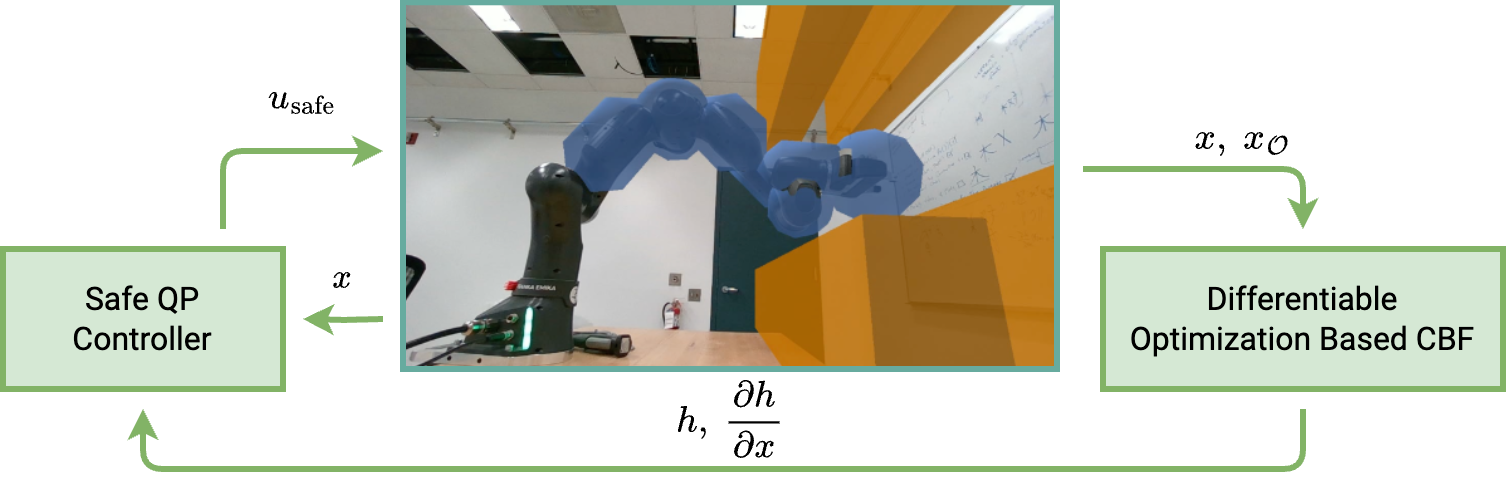}
    \vspace{-1em}
    \caption{The structure of safe robotic control using differentiable-optimization-based CBFs.}
    \label{fig:DiffOptCBF}
\end{figure}

One of the main hurdles to applying CBF-based methods to robotic systems is the construction of a valid CBF. Although work has been done in synthesizing CBFs for robots, there still lacks a systematic approach for CBF synthesis~\cite{NguyenHGAS16}. Given this difficulty, work has been done in learning CBFs from expert data of safe and unsafe interactions~\cite{RobeyHLZDTM20}. CBFs for robotics systems can also be learned online using onboard sensor data~\cite{LiZNLLB21}. Additionally, work has been done in learning CBFs starting from a handcrafted CBF~\cite{DaiKK22}. Although the learning-based methods for CBF synthesis are appealing, acquiring a dataset may be costly in many cases, e.g., self-driving cars. To simplify the CBF construction, work has been done in modeling the interaction between robots and the environment using points and objects~\cite{DaiKK22}, spheres~\cite{SingletaryKA22}, and points and higher-order ellipsoids~\cite{MurtazaAWH22}. However, since robotic systems consist of 3D geometrical entities, these approximations either completely ignore or over-estimate the robot geometry, which leads to over-conservative control policies. Recently, work has been done in constructing CBFs for robots represented as polytopic shapes~\cite{ThirugnanamZS22} and finding the safe control action using nonlinear model predictive control (NMPC). However, extending their formulation to other shapes, e.g., spheres and cylinders, is not straightforward. Another possible choice for constructing CBFs is to use a proximity measurement like the shortest distance. However, the shortest distance is not well defined when two objects overlap. To address this issue, signed distance functions (SDFs) are used instead in~\cite{SingletaryGMSA22} to construct CBFs. However, algorithms used to compute the SDF, e.g., the Gilbert–Johnson–Keerthi (GJK) distance algorithm~\cite{GilbertJK88} and the expanding polytope algorithm (EPA)~\cite{Bergen01}, are not differentiable~\cite{TracyHM22} and SDFs are also nonsmooth~\cite{SingletaryGMSA22}. Thus, an approximated version of the partial derivative of the SDF is used along with a looser constraint to generate safe control actions, which results in a conservative controller. 

We construct our CBF in this paper to overcome the above limitations based on differentiable optimization methods. Unlike traditional optimization solvers, differentiable optimization solvers, in addition to the optimal solution, also provide the partial derivatives of the optimal solution with respect to the problem parameters~\cite{AmosK17}. Given their differentiability, differentiable optimization solvers have gained popularity in the machine learning community by integrating them within a deep learning pipeline~\cite{AgrawalABBDK19}. Recently, differentiable optimization solvers have also been used for collision detection between two convex shapes~\cite{TracyHM22} by finding the minimum scaling factor for the two objects that leads to their collision.

This work builds upon the idea of using a differentiable collision detector and proposes a method to define CBFs for robotic obstacle avoidance tasks systematically. The proposed approach is efficient to compute, directly differentiable, handles a wide range of geometries, and is well-defined even in collision. The main contribution of this paper is twofold: (1) proposing and theoretically analyzing a differentiable optimization based approach to synthesize CBFs for robotic obstacle avoidance tasks that consider both the robot and scene geometry; (2) performing simulations and experiments (on FR3) to show the efficacy of our approach. This paper is structured as follows. In Section~\ref{sec:preliminaries}, we briefly review CBFs and mathematical foundations of differentiable optimization solvers. In Section~\ref{sec:problem_formulation}, we formulate the safe robotic control problem. In Section~\ref{sec:method}, we present our approach for constructing CBFs using differentiable optimization solvers. In Section~\ref{sec:experiments}, we show efficacy of our approach using a 2D mobile robot and on the seven degrees-of-freedom (DOF) FR3 robotic arm in both simulation and real world. Section~\ref{sec:conclusion} concludes the paper with a discussion on future directions.
\section{Preliminaries}
\label{sec:preliminaries}
This section presents a brief introduction to CBF and differentiable optimization.

\subsection{Control Barrier Functions}
Consider a control affine system
\begin{equation}
    \dot{x} = F(x) + G(x)u
    \label{eq:control_affine_sys}
\end{equation}
where the state is $x\in\mathbb{R}^n$ and the control input is $u\in\mathbb{R}^m$, with $\mathcal{U}$ being the admissible set of controls. The locally Lipschitz continuous functions $F: \mathbb{R}^n\rightarrow\mathbb{R}^n$ and $G: \mathbb{R}^n\rightarrow\mathbb{R}^{n\times m}$ represent the drift and the control influence matrix, respectively. Additionally, we assume access to a controller $u = \pi(x)$, with $\pi: \mathbb{R}^n\rightarrow\mathbb{R}^{m}$ being locally Lipschitz continuous. We say the controller $\pi$ can keep the system in~\eqref{eq:control_affine_sys} safe with respect to a set $\mathcal{C}\subset\mathbb{R}^n$ if the controller $\pi$ renders the set $\mathcal{C}$ forward control invariant. In other words, the controller $\pi$ keeps~\eqref{eq:control_affine_sys} safe with respect to $\mathcal{C}$, if for any initial state $x_0 \in \mathcal{C}$, the solution to~\eqref{eq:control_affine_sys}, defined as $x(t)$, remains within the safe set $\mathcal{C}$ $\forall t\in\mathbf{I}(x_0)$. The time interval of existence $\mathbf{I}(x_0) = [t_0, t_{\mathrm{max}})$ is where $x(t)$ is a unique solution to~\eqref{eq:control_affine_sys}; the system defined in~\eqref{eq:control_affine_sys} is considered forward complete when $t_{\mathrm{max}} = \infty$. Let the set $\mathcal{C}$, with $\mathcal{C} \subset \mathcal{D} \subset \mathbb{R}^n$, be the 0-superlevel set of a continuously differentiable function $\mathbf{h}:\mathcal{D} \rightarrow \mathbb{R}$ that has the property $\partial\mathbf{h}/\partial x \neq 0$ for all $x\in\partial\mathcal{C}$. Then, for~\eqref{eq:control_affine_sys}, if 
\begin{equation}
    \sup_{u\in\mathcal{U}}\Big[\frac{\partial\mathbf{h}(x)}{\partial x}\Big(F(x) + G(x)u\Big)\Big] \geq -\Lambda(\mathbf{h}(x))
    \label{eq:CBF_constraint}
\end{equation}
holds for all $x\in\mathcal{D}$, with $\Lambda: \mathbb{R}\rightarrow\mathbb{R}$ being an extended class $\mathcal{K}_\infty$ function\footnote{Extended class $\mathcal{K}_\infty$ functions are strictly increasing with $\Lambda(0) = 0$.}, we say that $\mathbf{h}$ is a CBF on $\mathcal{C}$. 

\subsection{Differentiable Optimization}
Consider a convex optimization problem in the form of
\begin{align}
    \min_{y\in\mathbb{R}^n}\ &\ f(y \mid \psi)\\
    \mathrm{subject\ to}\ &\ \ell(y \mid \psi) = 0\nonumber\\
                          &\ h(y \mid \psi) \leq 0\nonumber
\end{align}
with $f:\mathbb{R}^n\rightarrow\mathbb{R}$ being the convex objective function, $\ell:\mathbb{R}^n\rightarrow\mathbb{R}^{n_e}$ the affine equality constraints, $h:\mathbb{R}^n\rightarrow\mathbb{R}^{n_i}$ the convex inequality constraints, and $\psi\in\mathbb{R}^{n_\psi}$ the problem configuration parameters. $n_e$, $n_i$, and $n_\psi$ represent the number of equality constraints, inequality constraints, and problem configuration parameters, respectively. The Karush-Kuhn-Tucker (KKT) conditions for stationarity, primal feasibility, and complementary slackness are
\begin{subequations}
\label{eq:kkt}
\begin{align}
    0 &= \partial_y\Big(f(y^\star) + (\lambda^\star)^\top\ell(y^\star) + (\nu^\star)^\top h(y^\star)\Big)\\
    0 &= \ell(y^\star) \in \mathbb{R}^{n_e}\\
    0 &= \mathrm{diag}(\nu^\star)h(y^\star) \in \mathbb{R}^{n_i}
\end{align}
\end{subequations}
with $\lambda\in\mathbb{R}^{n_e}$ and $\nu\in\mathbb{R}^{n_i}$ being the dual variables, the superscript $\star$ representing the optimal values, and the $\psi$'s omitted for brevity. Define $z^\star(\psi) = \begin{bmatrix}
        y^\star(\psi) & \lambda^\star(\psi) & \nu^\star(\psi)
\end{bmatrix}^\top\in\mathbb{R}^{n_z}$, with $n_z = n + n_e + n_i$. Eq.~\eqref{eq:kkt} can be written as
\begin{equation}
    g(z^\star(\psi), \psi) = 0
    \label{eq:kkt_g}
\end{equation}
with $g:\mathbb{R}^{n_z}\times\mathbb{R}^{n_{\psi}}\rightarrow\mathbb{R}^{n_z}$. To compute the partial derivative of $z^\star$ w.r.t $\psi$, we utilize the implicit function theorem.

\begin{thm}[Implicit Function Theorem~\cite{Dini07}]
\label{thm:implicit_function_theorem}
Let $g:\mathbb{R}^{n_z + n_\psi} \rightarrow \mathbb{R}^{n_\psi}$ be a continuously differentiable function, $z_0^\star \in \mathbb{R}^{n_z}$, and $\psi_0 \in \mathbb{R}^{n_\psi}$. Let $(z_0^*, \psi_0)$ satisfy $g(z_0^\star, \psi_0) = 0$, and assume $\partial_{z^\star}g(z^\star, \psi_0)|_{z^\star = z_0^\star}$ is invertible. Then there exist open sets $S_{z^\star} \subset \mathbb{R}^{n_z}$ and $S_\psi \subset \mathbb{R}^{n_\psi}$ containing $z_0^\star$ and $\psi_0$, respectively, and a unique continuously differentiable function $z^\star: \mathbb{R}^{n_\psi} \rightarrow \mathbb{R}^{n_z}$ such that $z^\star(\psi_0) = z_0^\star$, $g(z^\star(\psi_0), \psi_0) = 0$. 
\end{thm}
Assume the implicit function Theorem holds for $g$, $z_0^\star$, and $\psi_0$. Taking the derivative on both sides of Eq.~\eqref{eq:kkt_g} yields  
\begin{equation}
    \frac{dg(z^\star, \psi)}{d\psi}\Bigg|_{\substack{\vspace{-1.5em} \\ {z^\star = z_0^\star} \\ {\psi = \psi_0}}} = 0.
    \label{eq:ift_1}
\end{equation}
Then, using the chain rule, we have
\begin{equation}
    \partial_{z^\star}g(z^\star, \psi_0)\Big|_{z^\star = z_0^\star}\partial_\psi z^\star(\psi)\Big|_{\psi = \psi_0} + \partial_\psi g(z_0^\star, \psi)\Big|_{\psi = \psi_0} = 0.
    \label{eq:ift_2}
\end{equation}
Finally, applying the implicit function Theorem yields
\begin{equation}
    \partial_\psi z^\star(\psi)\Big|_{\psi = \psi_0} = -\partial_{z^\star}^{-1}g(z^\star, \psi_0)\Big|_{z^\star = z_0^\star}\partial_\psi g(z_0^\star, \psi)\Big|_{\psi = \psi_0}.
    \label{eq:ift_3}
\end{equation}
This provides a way to efficiently and exactly compute the gradient of the optimal solution of a convex optimization problem with respect to its problem configuration.

\subsection{Berge's Maximum Theorem}
A key theoretical result we use later in the paper is Berge's Maximum Theorem.
\begin{thm}[Berge's Maximum Theorem~\cite{Border85}]
\label{thm:berge}
Let $X$ and $\Psi$ be topological spaces, and $\mathcal{J}: X\times\Psi \rightarrow \mathbb{R}$ be a continuous function on $X\times\Psi$ and $\Gamma: \Psi\rightrightarrows X$ be a compact-valued correspondence\footnote{Correspondences are set-valued functions.} such that $\Gamma(\psi) \neq \emptyset$ for all $\psi\in\Psi$. Define the value function $\mathcal{J}^\star: \Psi \rightarrow \mathbb{R}$ as
\begin{equation}
    \mathcal{J}^\star(\psi) = \sup\{\mathcal{J}(x, \psi) \mid x \in \Gamma(\psi)\}
\end{equation}
and the solution function $\Gamma^\star: \Psi \rightarrow X$ as
\begin{equation}
    \Gamma^\star(\psi) = \{x \mid x \in \Gamma(\psi), \mathcal{J}(x, \psi) = \mathcal{J}^\star(\psi)\}.
\end{equation}
If $\Gamma$ is an upper and lower hemicontinuous (UHC/LHC) correspondence~\cite{Border85} at $\psi$, then $\mathcal{J}^\star$ is continuous and $\Gamma^\star$ is UHC with nonempty and compact values.
\end{thm}

\section{Problem Formulation}
\label{sec:problem_formulation}

We consider the problem of CBF-based obstacle avoidance for robotic systems in the form of~\eqref{eq:control_affine_sys}. For each obstacle, we can represent the obstacle avoidance task as $c_i(x) \geq 0$, where $c_i:\mathbb{R}^{n} \rightarrow \mathbb{R}$ represents the $i$-th obstacle avoidance constraint. For each constraint function $c_i$, we can define its 0-superlevel set as $\mathcal{C}_i$. The safe set $\mathcal{C}$ for the robot is then the intersection of $\mathcal{C}_i$'s, i.e.,
\begin{equation}
    \mathcal{C} = \bigcap_{i=0}^{n_i}\mathcal{C}_i = \bigcap_{i=0}^{n_i}\big\{x\ |\ x\in\mathbb{R}^n, c_i(x) \geq 0\big\}.
\end{equation}
This paper aims to find a systematic way of constructing a CBF such that the robot can stay within $\mathcal{C}$.
\section{Differentiable Optimization Based CBFs}
\label{sec:method}
In this section, we present our proposed method of using differentiable optimization to compute CBFs for safe robotic control. First, we motivate our work by showing the limitations of SDF-based CBFs. Second, the CBF formulation is presented. Then, we show how to construct the CBF for robotic applications. Finally, we show how to construct the optimization problem for finding the safe control action.

\subsection{Motivation}
For maintaining safety in obstacle avoidance tasks, the most straightforward approach is to use a proximity measurement, i.e., SDF, for constructing the CBF, which for two convex objects $A$ and $B$ is defined as $\mathrm{SDF}(A, B) = \mathrm{distance}(A, B) - \mathrm{penetration}(A, B)$. Currently, the most efficient method to compute the SDF is using GJK to compute the distance and EPA to compute the penetration. Using GJK and EPA, we get the witness points $\hat{p}_A, \hat{p}_B \in \mathbb{R}^4$ and the vector of minimal translation $\hat{n}\in\mathbb{R}^4$ in homogeneous form. Then, the CBF can be written as~\cite{SchulmanDHLABPPGA14}
\begin{equation}
    \mathbf{h}_{\mathrm{SDF}}(x_A) = \hat{n}(x_A)^\top(F_A^W(x_A)\hat{p}_A(x_A) - F_B^W\hat{p}_B(x_A))
    \label{eq:sdf_cbf}
\end{equation}
with $F_A^W, F_B^W \in \mathbb{R}^{4\times4}$ denoting the poses of $A$ and $B$ and $x_A\in\mathbb{R}^n$ representing the state of $A$. In~\eqref{eq:sdf_cbf}, we assume $B$ is static. However, due to their logical control flows and pivoting, the output of GJK and EPA, i.e., $(\hat{p}_A, \hat{p}_B, \hat{n})$, is inherently non-differentiable~\cite{TracyHM22}. The work in~\cite{SingletaryGMSA22} bypasses this issue by only taking the partial derivative with respect to $F_A^W$ and treating the remainder terms as disturbance:
\begin{equation}
    \frac{\partial\mathbf{h}_{\mathrm{SDF}}}{\partial x_A} = \hat{n}(x_A)^\top\underbrace{\frac{\partial F_A^W}{\partial x_A}\hat{p}_A(x_A)}_{J_A(x_A)} + \delta(x_A)
\end{equation}
with $\delta: \mathbb{R}^n \rightarrow \mathbb{R}^n$ representing the remainder terms associated with the derivatives of $\hat{n}$, $\hat{p}_A$, and $\hat{p}_B$. Then, the safety constraint is written as
\begin{equation}
    \hat{n}(x_A)^\top J_A(x_A)\dot{x}_A \geq -\gamma\mathbf{h}_{\mathrm{SDF}} + 2J_\mathrm{max}\dot{x}_\mathrm{A, max}
\end{equation}
with $\|\delta(x)\|_\infty \leq 2\max_{x_A}\|J_A(x_A)\| = 2J_\mathrm{max}$ and $\|\dot{x}_\mathrm{A}\|_\infty = \dot{x}_\mathrm{A, max}$. For derivation details, please refer to~\cite{SingletaryGMSA22}. This is not a tight bound and therefore generates a conservative controller that only recovers a portion of the safe set. In the remainder of this section, we formulate a CBF that is directly differentiable and can recover the entire safe set.

\subsection{CBF Formulation}
\label{sec:cbf_formulation}
Inspired by~\cite{TracyHM22}, we construct a CBF using the minimum scaling of two convex objects under which they collide. The scaling $\alpha\in\mathbb{R}_+$ can be computed using a conic program
\begin{align}
\label{eq:differentiable_collision}
    \min_{p, \alpha}\ &\ \alpha\\
    \mathrm{subject\ to}\ &\ p \in \mathcal{P}_A(\alpha)\nonumber\\
                          &\ p \in \mathcal{P}_B(\alpha)\nonumber\\
                          &\ \alpha > 0\nonumber.
\end{align}
where $p \in \mathbb{R}^3$ represents a point and $P_A(\alpha)$ a set that contains the interior and surface of $A$ after scaling it uniformly using a scaling factor $\alpha$. $P_B(\alpha)$ is defined in the same manner. Denote the optimal value for $\alpha$ as $\alpha^\star$. If $\alpha^\star > 1$, then the two convex objects are not in collision. The two convex objects collide if $\alpha^\star \leq 1$. The optimal $p$ for~\eqref{eq:differentiable_collision}, denoted as $p^\star$, represents the point of intersection after scaling the two convex objects. A visual illustration of the solution to~\eqref{eq:differentiable_collision} can be found in Fig.~\ref{fig:concept_plot}. Then, we formulate the CBF as
\begin{equation}
    \mathbf{h}(x) = \alpha^\star(x) - \beta.
    \label{eq:diffopt_cbf}
\end{equation}
with $\beta\in\mathbb{R}$ and $\beta \geq 1$. Using Theorem~\ref{thm:implicit_function_theorem} and Eq.~\eqref{eq:ift_1}-\eqref{eq:ift_3}, we can obtain the Jacobian of $\alpha^\star$ as
\begin{equation}
    \frac{\partial\alpha^\star}{\partial(r_1, q_1, r_2, q_1)} = \frac{\partial\mathbf{h}}{\partial(r_1, q_1, r_2, q_1)}\in\mathbb{R}^{1 \times 14}
    \label{eq:solver_jacobian}
\end{equation}
where $r_1, r_2\in\mathbb{R}^3$ represents the positions of $A$ and $B$, $q_1, q_2\in\mathbb{R}^4$ represents the orientations in quaternions, and $\psi = (r_1, q_1, r_2, q_2) \in \Psi \subset \mathrm{SE}(3)\times\mathrm{SE}(3)$ represents joint pose of the two objects.

\begin{figure}[t!]
    \centering
    \includegraphics[width=0.4\textwidth]{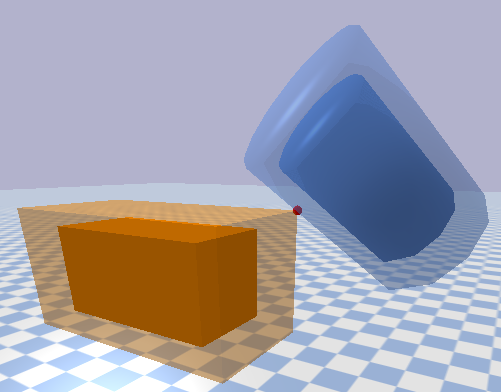}
    \caption{This figure illustrates the outcome of the differentiable conic program in~\eqref{eq:differentiable_collision}. The first convex object is the polygon and the second is the cylinder. The transparent regions surrounding the solid objects represent the scaled version of the objects after scaling them with a ratio of $\alpha$. The red dot represents the point of intersection $p$ of the two scaled objects.}
    \label{fig:concept_plot}
\end{figure}
To show~\eqref{eq:diffopt_cbf} is a valid CBF, we need to show continuous differentiability of $\alpha^\star$. In this paper, we prove continuity of $\alpha^\star$ in the general case and then prove continuous differentiability for the case of strongly convex scaling function when the gradients and Hessians exist and are continuous. We conjecture that continuous differentiability can be generalized to one object being strongly convex and the second being only convex although a formal proof is still elusive. 
\begin{thm}
\label{thm:main}
The optimal value of~\eqref{eq:differentiable_collision}, denoted as $\alpha^\star$, is a continuous function with respect to the positions and orientations of $A$ and $B$.
\end{thm}

To prove Theorem~\ref{thm:main}, we need  Lemma~\ref{lem:compact_correspondence} with the constraint correspondence $\Gamma(\psi)$ defined as
\begin{equation}
    \Gamma(\psi) = \{(p, \alpha)\mid p \in \mathcal{P}_{AB}(\alpha \mid \psi), \alpha \in S\}
    \label{eq:constraint_correspondence}
\end{equation}
with $S \subset \mathbb{R}$ being a compact set and $\mathcal{P}_{AB}(\alpha \mid \psi) = \mathcal{P}_{A}(\alpha \mid \psi) \cap \mathcal{P}_B(\alpha \mid \psi)$. The term $P_A(\alpha \mid \psi)$ denotes that $P_A$ depends on both $\alpha$ and $\psi$, where $\psi$ is seen as the configuration parameter. $\mathcal{P}_{AB}(\alpha \mid \psi)$ and $\mathcal{P}_B(\alpha \mid \psi)$ are defined in a similar manner as $\mathcal{P}_A(\alpha \mid \psi)$. Also we define $\Gamma_i(\psi) = \{(p,\alpha) \mid p\in\mathcal{P}_i(\alpha\mid\psi), \alpha \in S\}$ with $i \in \{A, B\}$. We can then show that $\Gamma(\psi)$ is a compact-valued correspondence on $\mathbb{R}^4$.

\begin{lem}
\label{lem:compact_correspondence}
The constraint correspondence $\Gamma(\psi)$ defined in~\eqref{eq:constraint_correspondence} is a compact-valued correspondence.
\end{lem}
\begin{proof}
First, we show that $\Gamma_A(\psi)$ is a compact-valued correspondence. Define $K_A(\psi) = \mathcal{P}_A(1\mid\psi) \subset \mathbb{R}^3$, and the function $\phi: K_A(\psi) \times S \rightarrow \mathbb{R}^4$ as $\phi(p, \alpha) = (\alpha(p - r_1) + r_1, \alpha)$, which is a continuous function on $K_A(\psi) \times S$. Because both $K_A(\psi)$ and $S$ are compact sets, we have $K_A(\psi) \times S$ as a compact set on $\mathbb{R}^4$. Since the image of $\psi$ under $\Gamma_A$ is the same as the image of $K_A(\psi) \times S$ under $\phi$, using the property that continuous functions map compact sets to compact sets, and $K_A(\psi)$ is a compact set for any value of $\psi$, we have $\Gamma_A(\psi)$ being a compact-valued correspondence. Using the same logic, for all $\psi$, $\Gamma_B(\psi)$ is also a compact-valued correspondence. Since $\Gamma(\psi) = \Gamma_A(\psi) \cap \Gamma_B(\psi)$, and intersections of compact sets are compact sets, we can conclude that $\Gamma(\psi)$ is a compact-valued correspondence.
\end{proof}

We can directly apply Theorem~\ref{thm:berge} to prove Theorem~\ref{thm:main} if we show that $\Gamma(\psi)$ is both UHC and LHC. 

\begin{proof}
\textbf{(Theorem~\ref{thm:main})} To show UHC, we consider an arbitrary open set $V$ such that  $\Gamma(\psi) \subset V$, i.e.,
\begin{equation}
    V = \Big\{(p, \alpha) \mid \mathcal{P}_{AB}(\alpha \mid \psi) \subset V_p(\alpha), S \subset V_\alpha\Big\}
\end{equation}
with $V_p$ being an open set on $\mathbb{R}^3$, $p \in \mathcal{P}_{AB}(\alpha \mid \psi)$, $V_\alpha$ being an open set on $\mathbb{R}$, and $\alpha \in S$. In our case, UHC is equivalent to: for any $\psi$ and $\alpha$, if the intersection of the two convex shapes $\mathcal{P}_A$ and $\mathcal{P}_B$ is contained within an open set $V_p(\alpha)$, then after applying an infinitesimally small translation and rotation, the intersection is still contained within $V_p(\alpha)$. Since $V_p(\alpha)$ is an open set, $\forall p \in \mathcal{P}_{AB}(\alpha \mid \psi)$ there exists an open neighborhood around $p$ that is a subset of $V_p(\alpha)$. Additionally, the compactness of $\mathcal{P}_{AB}(\alpha \mid \psi)$ implies that a finite subset of the union of those neighborhoods covers $\mathcal{P}_{AB}(\alpha \mid \psi)$. After applying an infinitesimally small transformation to $p$, it should stay in that finite cover because the change in distance between the origin and the point $p$ is Lipschitz continuous~\cite{GilbertJ85}. For LHC, we consider an arbitrary open set $\bar{V}$ such that $\Gamma(\psi) \cap \bar{V} \neq \emptyset$, i.e., with $p \in \mathcal{P}_{AB}(\alpha \mid \psi)$, $\bar{V}_p(\alpha)$ being an open set on $\mathbb{R}^3$, $\bar{V}_\alpha$ an open set on $\mathbb{R}$, and $\alpha \in S$, we have
\begin{equation}
    \bar{V} = \Big\{(p, \alpha) \mid \bar{V}_\alpha \cap S \neq \emptyset, \mathcal{P}_{AB}(\alpha \mid \psi) \cap \bar{V}_p(\alpha) \neq \emptyset\Big\}.
\end{equation}
In our case, LHC is equivalent to: for $\alpha \in \bar{V}_\alpha$, if $\mathcal{P}_{AB}(\alpha \mid \psi) \cap \bar{V}_p(\alpha) \neq \emptyset$, then after applying an infinitesimally small translation and rotation, the intersection set is not empty. Since there exists $p \in \mathcal{P}_{AB}(\alpha \mid \psi)$ with an open neighborhood that belongs to $\bar{V}_p(\alpha)$, following the logic for showing UHC, $\Gamma$ is also LHC. Thus, since $\Gamma(\psi)$ is a compact-valued correspondence, using Theorem~\ref{thm:berge}, we know $J^\star = \alpha^\star$ is a continuous function with respect to $\psi$.
\end{proof}
Next, under sufficient conditions, we show that $\alpha^\star$ is continuously differentiable with respect to $\psi$.
\begin{defn}[Scaling Function]
\label{def:scaling_func}
    We say that $\mathcal{F}_A: \mathbb{R}^3 \rightarrow \mathbb{R}$,  is a scaling function for an object $A \subset \mathbb{R}^3$, if it is a continuously differentiable convex function, $\mathcal{F}_A(p) \geq 0$, $A = \{p \mid \mathcal{F}_A(p) \leq 1\}$, $\partial A = \{p \mid \mathcal{F}_A(p) = 1\}$, and $\min_p\mathcal{F}_A(p) = 0$.
\end{defn}
\noindent For example, the scaling function for an ellipsoid $A$ has form $\mathcal{F}_A = (p - r)^\top\mathbf{P}(p - r)$ where $\mathbf{P} \in \mathbb{R}^3$ is a diagonal matrix with inverses of squared semi-axes lengths on its diagonal. This particular example is used in~\cite{TracyHM22} to implement the scaling constraint for the ellipsoid, and similar functions can be derived for other shapes discussed in~\cite{TracyHM22}.
\begin{lem}
\label{lem:scaling}
    For two objects $A, B \subset \mathbb{R}^3$, denote their scaling functions as $\mathcal{F}_A$ and $\mathcal{F}_B$, respectively. If $\mathcal{F}_A$ and $\mathcal{F}_B$ are differentiable functions, and no point $p$ exists such that $\partial\mathcal{F}_A/\partial p = \partial\mathcal{F}_B/\partial p = \mathbf{0}$, then for the optimal solution $(p^\star, \alpha^\star)$ of~\eqref{eq:differentiable_collision}, we have $\mathcal{F}_A(p^\star) = \mathcal{F}_B(p^\star) = \alpha^\star$ and $\nu_A^\star, \nu_B^\star \neq 0$.
\end{lem}
\begin{proof}
    First, we rewrite the problem in~\eqref{eq:differentiable_collision} as
    \begin{align}
    \label{eq:scaling_opt}
        \min_{p, \alpha}\ &\ \alpha\\
        \mathrm{subject\ to}\ &\ \mathcal{F}_A(p) \leq \alpha\nonumber\\
                              &\ \mathcal{F}_B(p) \leq \alpha\nonumber
    \end{align}
    We can write the Lagrangian for~\eqref{eq:scaling_opt} as
    \begin{equation}
        \mathcal{L}(\alpha, p, \nu_A, \nu_B) = \alpha + \nu_A(\mathcal{F}_A(p) - \alpha) + \nu_B(\mathcal{F}_B(p) - \alpha)
    \end{equation}
    From the stationarity condition of the KKT conditions we have
    \begingroup
    \allowdisplaybreaks
    \begin{subequations}
    \begin{align}
        \label{eq:lem_kkt_1}
        \frac{\partial\mathcal{L}}{\partial\alpha} &= 1 - \nu_A^\star - \nu_B^\star = 0\\
        \label{eq:lem_kkt_2}
        \frac{\partial\mathcal{L}}{\partial p} &= \nu_A^\star\frac{\partial\mathcal{F}_A}{\partial p}(p^\star) + \nu_B^\star\frac{\partial\mathcal{F}_B}{\partial p} = 0
    \end{align}
    \end{subequations}
    \endgroup
    From the complementary slackness condition we have
    \begin{subequations}
    \begin{align}
        \label{eq:lem_kkt_3}
        \nu_A^\star(\mathcal{F}_A(p^\star) - \alpha^\star) &= 0\\
        \label{eq:lem_kkt_4}
        \nu_B^\star(\mathcal{F}_B(p^\star) - \alpha^\star) &= 0.
    \end{align}
    \end{subequations}
    It can be shown that neither $\nu_A^\star$ nor $\nu_B^\star$ can be 0. To show this, without loss of generality, assume $\nu_A^\star = 0$. From~\eqref{eq:lem_kkt_1}, we have $\nu_B^\star = 1$ and from~\eqref{eq:lem_kkt_2} we have $\nabla_p\mathcal{F}_B(p^\star) = 0$. From definition~\ref{def:scaling_func}, we know that $\mathcal{F}_B$ is a convex differentiable function, therefore, $p^\star$ is the global minimum of $\mathcal{F}_B$. Hence, also from definition~\ref{def:scaling_func}, $\mathcal{F}_B(p^\star) = 0$. Then, from~\eqref{eq:lem_kkt_4}, $\alpha^{\star} = \mathcal{F}_B(p^\star) = 0$. Using the constraint $\mathcal{F}_A(p) \leq \alpha$ from~\eqref{eq:scaling_opt} and the condition $\mathcal{F}_A(p) \geq 0$ from definition~\eqref{def:scaling_func}, we have $\mathcal{F}_A(p^\star) = 0$, which means $p^\star$ is also a global minimizer for $\mathcal{F}_A$, i.e., $\partial\mathcal{F}_A(p^\star)/\partial p = \mathbf{0}$. Since, by the assumption of Lemma~\ref{lem:scaling}, no point $p$ exists such that $\partial\mathcal{F}_A/\partial p = \partial\mathcal{F}_B/\partial p = \mathbf{0}$, thus, $\nu_A^\star \neq 0$. Given the symmetry between $\nu_A^\star$ and $\nu_B^\star$, $\nu_B^\star \neq 0$. Therefore, from~\eqref{eq:lem_kkt_3} and~\eqref{eq:lem_kkt_4}, we have $\mathcal{F}_A(p^\star) = \mathcal{F}_B(p^\star) = \alpha$.
\end{proof}

\begin{thm}
\label{thm:cont_diff}
    Consider two parameterized scaling functions $\mathcal{G}_A(p,\psi) = \mathcal{F}_{A(\psi)}(p)$ and $\mathcal{G}_B(p,\psi) = \mathcal{F}_{B(\psi)}(p)$ such that, for any $\psi \in \Psi$, these are scaling functions  for varying sets $A(\psi), B(\psi) \subset \mathbb{R}^3$. Furthermore, for all $p$ and $\psi$, assume that the conditions of Lemma~\ref{lem:scaling} holds, $\mathcal{G}_A$ and $\mathcal{G}_B$ are strongly convex on $p$, the gradients $\partial\mathcal{G}_A/\partial\psi$ and $\partial\mathcal{G}_B/\partial\psi$ and the Hessians $\partial^2\mathcal{G}_A/\partial p^2$, $\partial^2\mathcal{G}_A/\partial p\partial\psi$, $\partial^2\mathcal{G}_B/\partial p^2$, and $\partial^2\mathcal{G}_B/\partial p\partial\psi$ are continuous in $\psi$, then, $\alpha^\star(\psi)$ is continuously differentiable w.r.t. $\psi$.
\end{thm}
\begin{proof}
    From~\eqref{eq:lem_kkt_1}, we have
        $\partial\nu_A^\star/\partial\psi + \partial\nu_B^\star/\partial\psi = 0$.
    Then, using Lemma~\ref{lem:scaling}, and taking the partial derivative with respect to $\psi$ on both sides of~\eqref{eq:lem_kkt_1} and~\eqref{eq:lem_kkt_2}, we obtain
    \begin{equation}
    \label{eq:linear_equations}
        \underbrace{\begin{bmatrix}
            \mathbf{M} & \mathbf{c}\\
            \mathbf{c}^\top & \mathbf{0}
        \end{bmatrix}}_{\mathbf{N}}\begin{bmatrix}
            \displaystyle\Big(\frac{\partial p^\star}{\partial\psi}\Big)^\top & \displaystyle\Big(\frac{\partial\nu_A^\star}{\partial\psi}\Big)^\top
        \end{bmatrix}^\top = \begin{bmatrix}
            \Omega_1\\
            \Omega_2
        \end{bmatrix}
    \end{equation}
    with
    \begin{subequations}
    \begin{align}
        \label{eq:thm_eq_1}
        \mathbf{M} &= \frac{\partial^2\mathcal{G}_A}{\partial p^2}(p^\star, \psi)\nu_A^\star + \frac{\partial^2\mathcal{G}_B}{\partial p^2}(p^\star, \psi)\nu_B^\star\\
        \label{eq:thm_eq_2}
        \mathbf{c} &= \frac{\partial\mathcal{G}_A}{\partial p}(p^\star, \psi) - \frac{\partial\mathcal{G}_B}{\partial p}(p^\star, \psi)\\
        \label{eq:thm_eq_3}
        \Omega_1 &= -\frac{\partial^2\mathcal{G}_A}{\partial p\partial\psi}(p^\star, \psi)\nu_A^\star - \frac{\partial^2\mathcal{G}_B}{\partial p\partial\psi}(p^\star, \psi)\nu_B^\star\\
        \label{eq:thm_eq_4}
        \Omega_2 &= \frac{\partial\mathcal{G}_B}{\partial\psi}(p^\star, \psi) - \frac{\partial\mathcal{G}_A}{\partial\psi}(p^\star, \psi).
    \end{align}
    \end{subequations}
    Since $\mathcal{G}_A$ and $\mathcal{G}_B$ are strongly convex and $\nu_A^\star$ and $\nu_B^\star$ are dual feasible and thus strictly positive from Lemma~\ref{lem:scaling}, $\mathbf{M} \succ 0$. Note that $\mathbf{c} = \mathbf{0}$ would imply from \eqref{eq:lem_kkt_1} and \eqref{eq:lem_kkt_2} that $\partial\mathcal{G}_A(p^\star, \psi)/ \partial p = \partial\mathcal{G}_B(p^\star, \psi)/ \partial p = 0$, violating the condition of Lemma~\ref{lem:scaling}. Hence, $\mathbf{c} \neq \mathbf{0}$. From Schur's formula, we know that
    \begin{equation}
        \det(\mathbf{N}) = \det(\mathbf{M})\det(-\mathbf{c^\top M^{-1}c}).
    \end{equation}
    Since $\mathbf{M} \succ 0$ and $\mathbf{c} \neq \mathbf{0}$, we can conclude that $\det(\mathbf{N}) \neq 0$, i.e., $\mathbf{N}$ is invertible. Since, $\mathbf{M}$, $\mathbf{c}$, $\Omega_1$, and $\Omega_2$ are continuous in $\psi$, $\mathbf{N}^{-1}$ is continuous on $\psi$, thus, $\partial p^\star/\partial\psi$ is also continuous in $\psi$. Then, from Lemma~\ref{lem:scaling}, we have
    \begin{equation}
        \alpha^\star = \mathcal{G}_A(p^\star, \psi).
    \end{equation}
    Taking the derivative w.r.t. $\psi$ on both sides of the equation yields
    \begin{equation}
        \frac{\partial\alpha^\star}{\partial\psi}(\psi) = \frac{\partial\mathcal{G}_A}{\partial p}(p^\star, \psi)\frac{\partial p^\star}{\partial\psi}(\psi) + \frac{\partial\mathcal{G}_A}{\partial\psi}(p^\star, \psi)
    \end{equation}
    which, following the above derivation, is continuous.
\end{proof}
\begin{rem}
    For ellipsoids/spheres, we can write a scaling function (as noted after Definition~\ref{def:scaling_func}) that satisfies Theorem~\ref{thm:cont_diff}.
\end{rem}

From empirical studies, we believe that continuous differentiability holds under more general conditions motivating Conjecture~\ref{conj:cbf_condition}.

\begin{conj}
\label{conj:cbf_condition}
Assume we have two convex objects $A$ and $B$. If $A$ is a strongly convex object (e.g., ellipsoid, sphere) and $B$ is a convex object that can either have a non-smooth surface (e.g., cylinder, polygon, cone) or a smooth surface, then the CBF defined in~\eqref{eq:diffopt_cbf} is continuously differentiable.
\end{conj}

\subsection{CBF Constraint}
\label{sec:cbf_constraint}
We consider the CBF formulation in Section~\ref{sec:cbf_formulation} using velocity control. For robotic systems with dynamics~\eqref{eq:control_affine_sys}, it was shown in~\cite{MolnarCSUA22} that a derivative controller with gravity compensation could realize input-to-state safety (ISSf) when tracking a planned safe velocity command. Similar to~\cite{MolnarCSUA22}, by tuning the value of $\beta$ in~\eqref{eq:diffopt_cbf}, we can achieve safety with respect to the actual safe set, even when the velocity command is not perfectly tracked. For velocity control, the system dynamics have the form of $\dot{x} = G(x)u$, with $n_j$ denoting the number of generalized coordinates and $n = m = n_j$. The above equation is the same as~\eqref{eq:control_affine_sys} when $F(x) = 0$. For the CBF constraint, we need to compute the value of $\dot{\mathbf{h}}(x)$. For a single rigid body, the chain rule yields
\begin{equation}
    \dot{\mathbf{h}}(x) = \frac{\partial\mathbf{h}}{\partial\mu}\frac{\partial\mu}{\partial x}\dot{ x}.
\end{equation}
where $\mu = [r^\top, q^\top]^\top\in\mathbb{R}^7$ represents the rigid body's position and orientation. The first term is obtained from~\eqref{eq:solver_jacobian}. The second term can be separated into two parts: the positional part and the orientational part. For the positional part, $\partial r / \partial x = J_v(x)$, where $J_v: \mathbb{R}^{n_j} \rightarrow \mathbb{R}^{3 \times n_j}$ is the positional Jacobian matrix. The orientational Jacobian $J_\omega$ has the relationship $\omega = J_\omega(x)\dot{x}$, where $\omega\in\mathbb{R}^3$ represents the frame angular velocity and $J_\omega: \mathbb{R}^{n_j} \rightarrow \mathbb{R}^{3 \times n_j}$. Define the (vectorized) quaternion as $q = \begin{bmatrix} 
        q_w & q_x & q_y & q_z
\end{bmatrix}^\top \in \mathbb{R}^4$, which satisfies
\begin{align}
    \dot{q} &= \frac{1}{2}\mathbf{Q}\omega = \frac{1}{2}\mathbf{Q}J_\omega(x)\dot{x}\\
    \mathbf{Q} &= \begin{bmatrix}
        -q_x & -q_y & -q_z\\
        q_w & -q_z & q_y\\
        q_z & q_w & -q_x\\
        -q_y & q_x & q_w
    \end{bmatrix}\in\mathbb{R}^{4\times3}.
\end{align}
Then, we have
\begin{equation}
    \frac{\partial\mu}{\partial x} = \begin{bmatrix}
        \displaystyle\frac{\partial r}{\partial x}\vspace{0.5em}\\
        \displaystyle\frac{\partial q}{\partial x}
    \end{bmatrix} = \begin{bmatrix}
        J_v(x)\vspace{0.25em}\\
        \displaystyle\frac{1}{2}\mathbf{Q}J_\omega(x)
    \end{bmatrix}\in\mathbb{R}^{7 \times n_j}.
\end{equation}
Finally, we have the CBF constraint as
\begin{equation}
    \frac{\partial\mathbf{h}}{\partial x}G(x)u = \frac{\partial\mathbf{h}}{\partial\mu}\frac{\partial\mu}{\partial x}G(x)u \geq -\gamma\mathbf{h}(x).
    \label{eq:final_cbf_constraint}
\end{equation}
with $\gamma\in\mathbb{R}_+$. In the CBF constraint, we use $\Lambda(a) = \gamma a$.

\begin{rem}
Since the Jacobians are continuous functions, the partial derivative of $\alpha$ w.r.t. the general coordinates of the robot is a continuous function on $\mathcal{D}$ if the partial derivative of $\alpha$ w.r.t. $(r, q)$ is a continuous function on $\mathcal{D}$.
\end{rem}

In the remainder of this section, we show conditions that guarantee $\partial\mathbf{h}/\partial x \neq 0$ for $x\in\partial\mathcal{C}$, which guarantees the validity of the CBF.

\begin{lem}
\label{lem:cbf_constraint}
For $x\in\partial\mathcal{C}$, if the Jacobian matrix $J(x) = \begin{bmatrix}
        J_v^\top(x) & J_\omega^\top(x)
\end{bmatrix}^\top$ has full row rank, then $\partial\mathbf{h}/\partial x \neq \mathbf{0}$.
\end{lem}
\begin{proof}
We can write $\partial\mathbf{h}/\partial x$ as
\begin{equation}
    \frac{\partial\mathbf{h}}{\partial x} = \frac{\partial\mathbf{h}}{\partial\mu}\begin{bmatrix}
        \mathbf{I} & \mathbf{0}\\
        \mathbf{0} & \displaystyle\frac{1}{2}\mathbf{Q}
    \end{bmatrix}J(x) = \frac{\partial\mathbf{h}}{\partial\mu}\mathbf{A}_\mathbf{Q}J(x).
\end{equation}
It can be seen that if $(\partial\mathbf{h}/\partial\mu)\mathbf{A}_\mathbf{Q} = 0$, then no infinitesimally small change in pose (including translations and rotations) exists that changes $\mathbf{h}$. It is evident that this is not the case since the distance between two convex objects can always be modified by some combination of translations and rotations. Thus, we know that $(\partial\mathbf{h}/\partial\mu)\mathbf{A}_\mathbf{Q} \neq 0$. Since $J(x)$ has full row rank, we have $\partial\mathbf{h}/\partial x \neq \mathbf{0}$.
\end{proof}

\begin{rem}
In addition to Lemma~\ref{lem:cbf_constraint}, for a fully actuated system with no input constraint, there always exists a $u$ that satisfies the CBF constraint when $J(x)$ is non-singular. 
\end{rem}

\subsection{Safe Controller}

For robotic applications, robots and obstacles are often rigid multibody systems. We represent each rigid body using convex primitive shapes. Then, using the proposed CBF construction method, we write a CBF for each robot-obstacle-rigid-body pair within the robot's workspace. Let the robot and obstacles be segmented into $n_\mathcal{R}$ and $n_\mathcal{O}$ convex shapes, respectively. Then, we can construct $n_\mathcal{R} \times n_\mathcal{O}$ CBFs, and the same number of CBF constraints, one for each robot-obstacle-rigid-body pair. Then, we combine all the CBF constraints and write it as an element-wise inequality
\begin{equation}
    \frac{\partial\mathbf{H}}{\partial x}G(x)u \geq -\gamma\mathbf{H}(x), \ \ \mathbf{H} = \begin{bmatrix}
        \mathbf{h}_{i \times j}(x)
    \end{bmatrix}\in\mathbb{R}^{n_\mathcal{R} \times n_\mathcal{O}}
\end{equation}
where $i = 1, \cdots, n_\mathcal{R}$, $j = 1, \cdots, n_\mathcal{O}$, and $\mathbf{h}_{i \times j}(x)$ representing the CBF between the $i$-th robot segment and the $j$-th obstacle segment. CBF-based quadratic programs (CBFQPs) are commonly used in CBF-based methods to obtain safe control actions. CBFQP utilizes a performance controller $\pi_\mathrm{perf}$ to generate a reference control $u_\mathrm{ref}$. Then, the CBFQP acts as a safety filter that alters the possibly unsafe $u_\mathrm{ref}$ in a minimally invasive fashion to find its safe counterpart, i.e.,
\begin{align}
\label{eq:cbf_opt_perf}
    \min_{u}\ &\ \|u - u_\mathrm{ref}\|_2^2\\
    \mathrm{subject\ to}\ &\ \frac{\partial\mathbf{H}}{\partial x}G(x)u \geq -\gamma\mathbf{H}(x).\nonumber
\end{align}
This method works well when the performance controller is non-optimization-based, e.g., PID or control Lyapunov function (CLF) based controllers. For optimization-based controllers, we add the CBF constraints and solve it as a single optimization problem
\begin{align}
\label{eq:cbf_opt}
    \min_{u}\ &\ \mathcal{J}(u)\\
    \mathrm{subject\ to}\ &\ \frac{\partial\mathbf{H}}{\partial x}G(x)u \geq -\gamma\mathbf{H}(x)\nonumber
\end{align}
where $\mathcal{J}: \mathbb{R}^m \rightarrow \mathbb{R}$ is the objective function. We will demonstrate the use of these two control methods in Section~\ref{sec:experiments}. The controller computation flow is shown in Fig.~\ref{fig:DiffOptCBF}. Note that for two obstacles at similar distances with different scales, the CBF values would be smaller for the larger obstacle. However, the CBFQP controller would not favor one obstacle over the other since a feasible solution of the CBFQP would need to satisfy the CBF constraint of each individual obstacle.

\section{Experiments}
\label{sec:experiments}
This section shows the efficacy of our proposed approach. First, we show the performance of our method on a simulated mobile robot example. Then, we show the application of our method to a 7-degree-of-freedom (DOF) robot manipulator both in simulation and real life. In all our experiments, we set $\beta = 1.03$ and $\gamma = 5.0$, though a larger $\beta$ value would also work. All experiments are performed on a PC with 32GB of RAM and an Intel Core i7 11700 processor.

\begin{figure}[t!]
    \centering
    \includegraphics[width=0.49\textwidth]{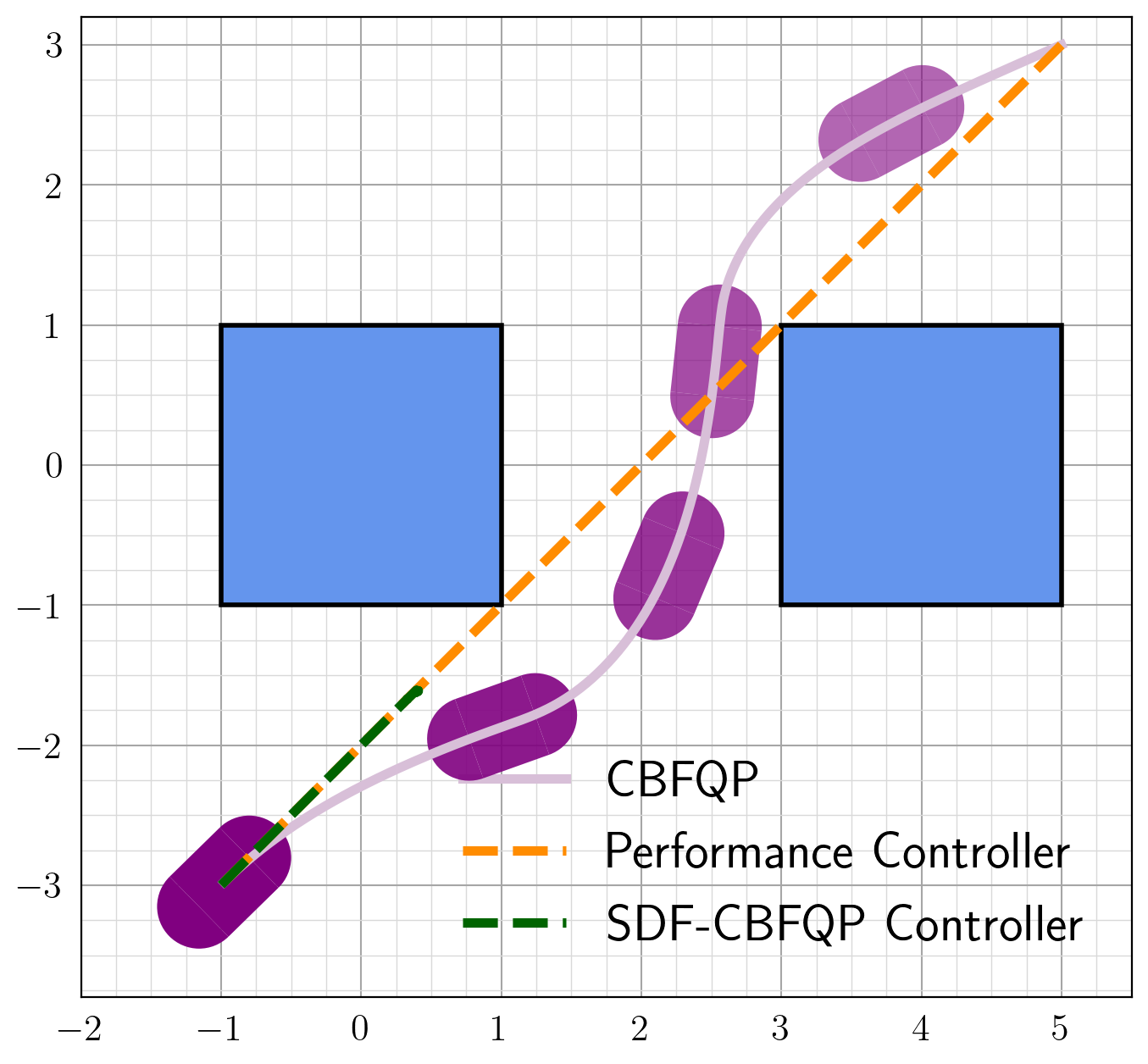}
    \caption{Generated motion for the mobile robot (purple capsule) in presence of obstacles (blue boxes).  The orange dashed curve is the motion generated by the performance controller $u_\mathrm{ref}$. The light purple curve represents the generated trajectory of the proposed CBFQP and the green dashed curve represents the generated trajectory of the SDF-CBFQP controller~\cite{SingletaryKA22} which prematurely terminates (i.e., without reaching the target state).}
    \label{fig:toy_example}
\end{figure}

\subsection{Mobile Robot Example}

\begin{figure*}[t!]
    \centering
    \includegraphics[width=\textwidth]{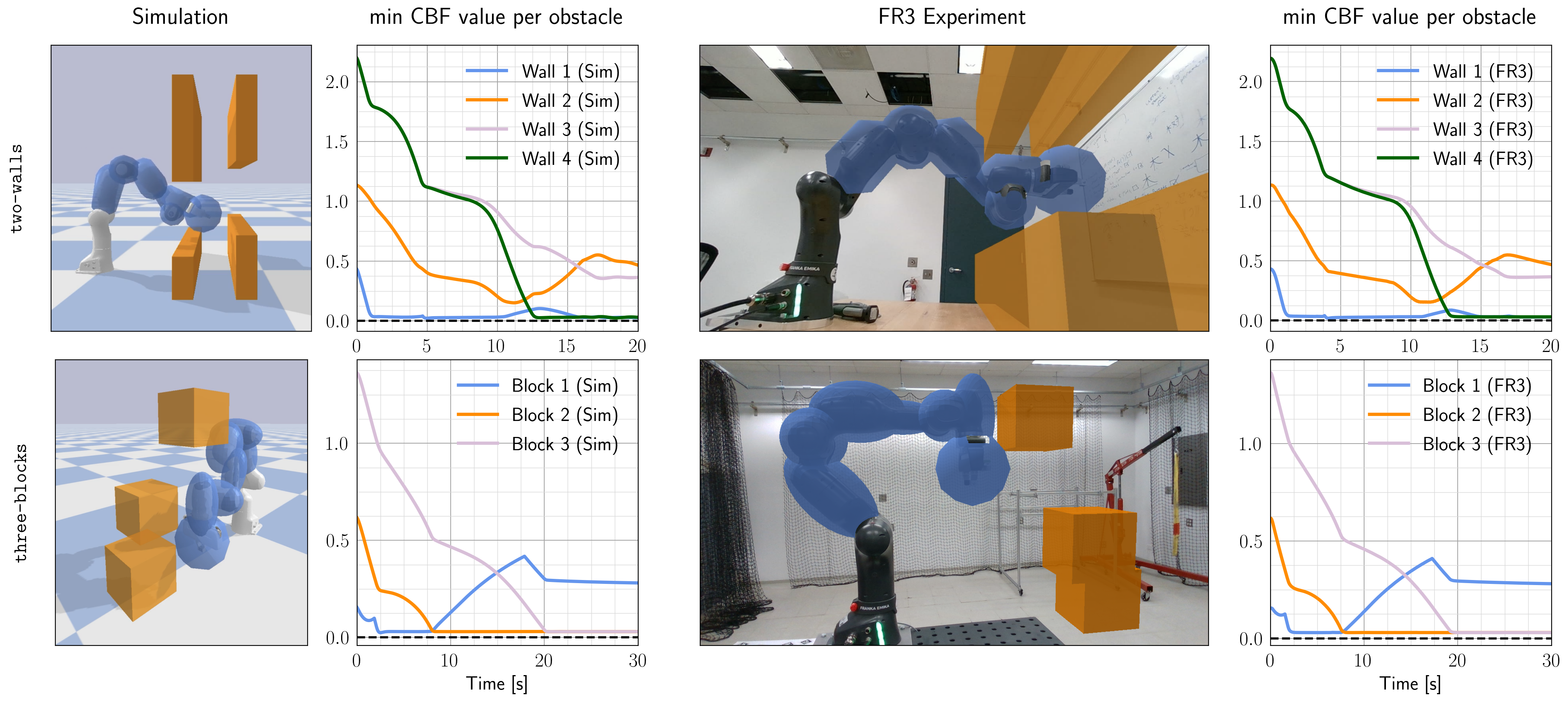}
    \vspace{-1em}
    \caption{This figure shows the proposed CBF-based obstacle avoidance controller can ensure safety on an FR3 robot, both in simulation and on the real robot. The simulation is performed using PyBullet with a time step of 1 ms. The walls are numbered in the front upper, front lower, back upper, and back lower order. The blocks are numbered in the top, middle, and bottom order. For each experiment, we record the CBF values of each link-obstacle pair, and for each obstacle, we show the minimum CBF value among all the links. For the real robot experiments, we synthetically show the bounding boxes as blue overlays.}
    \label{fig:manipulator_exp_results}
\end{figure*}

\begin{figure}[t!]
    \centering
    \includegraphics[width=0.49\textwidth]{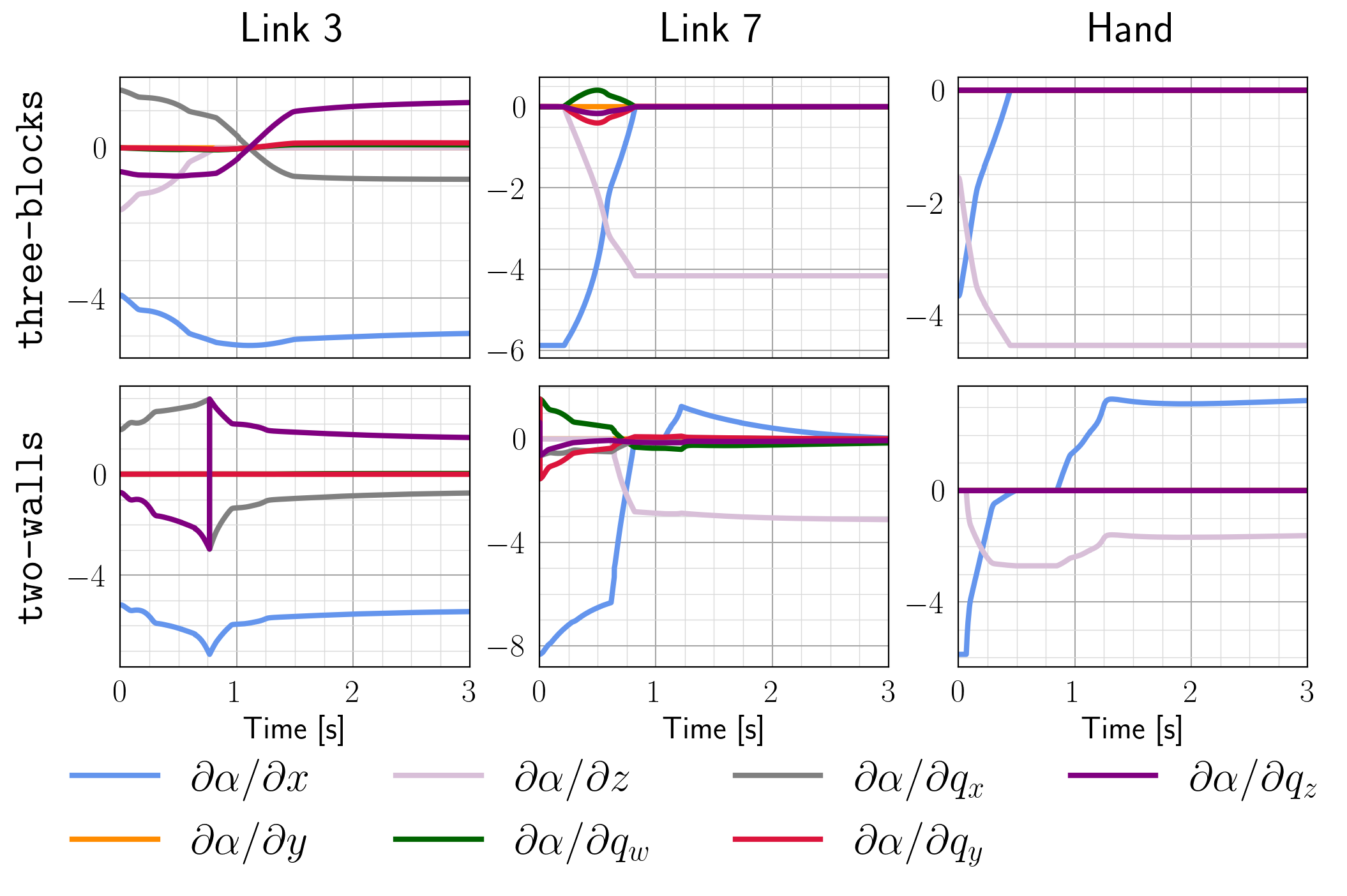}
    \vspace{-1em}
    \caption{The computed partial derivatives of the CBF in the real-world robot experiment on the \texttt{three-blocks} and \texttt{two-walls} task. The partial derivatives for link 4, link 5, and link 6 are similar to link 7, thus, in interest of space, they are not shown in the figure above.}
    \label{fig:jacobians_of_cbf_paper}
\end{figure}

We demonstrate the efficacy of our approach on a mobile robot example having a capsule shape and the obstacles are represented as polygons. We assume that we can directly control the robot's linear and angular velocity, i.e.,
\begin{equation}
    \begin{bmatrix}
        \dot{p}_x & \dot{p}_y & \dot{\phi}
    \end{bmatrix}^\top = \begin{bmatrix}
        v\cos\phi\\
        v\sin\phi\\ \omega
    \end{bmatrix}.
\end{equation}
The area the robot occupies after scaling it with $\alpha$ can be described using the constraints~\cite{TracyHM22}:
\begingroup
\allowdisplaybreaks
\begin{subequations}
\label{eq:toy_example_robot_constraint}
\begin{align}
    \begin{bmatrix}
        0\\
        0
    \end{bmatrix} - \begin{bmatrix}
        0_{1\times3} & -l/2 & 1\\
        0_{1\times3} & -l/2 & -1\\
    \end{bmatrix}\begin{bmatrix}
        p\\
        \alpha\\
        \delta
    \end{bmatrix} &\in \mathbb{R}_+^2\\
    \begin{bmatrix}
        0\\
        -r
    \end{bmatrix} - \begin{bmatrix}
        0_{1\times3} & -R & 0\\
        -I_{3\times3} & 0_{1\times3} & \hat{b}_x\\
    \end{bmatrix}\begin{bmatrix}
        p\\
        \alpha\\
        \delta
    \end{bmatrix} &\in \mathcal{Q}_4
\end{align}
\end{subequations}
\endgroup
with $r\in\mathbb{R}^3$ being the position of the capsule, $\mathcal{Q}_4 \subset \mathbb{R}^4$ the second-order cone, $R\in\mathbb{R}_+$ the radius of the capsule, $L\in\mathbb{R}_+$ the length of the capsule, $l\in\mathbb{R}_+$ the line segment distance, $\delta\in[-\alpha L/2, \alpha L/2]$ a slack variable, $\hat{b}_x = \mathbf{R}[1, 0, 0]^\top$, and $\mathbf{R}$ being the rotation matrix representing the orientation of the robot. For any point $p$ that satisfies~\eqref{eq:toy_example_robot_constraint}, it will belong to the scaled version of the robot body. The area occupied by the scaled version of the obstacles is represented using the constraint~\cite{TracyHM22}:
\begin{equation}
    A_oR_or_o - \begin{bmatrix}
        A_oR_o^\top & -b_o
    \end{bmatrix}\begin{bmatrix}
        p\\
        \alpha
    \end{bmatrix} \in \mathbb{R}_+,
\label{eq:toy_example_obstacle_constraint}
\end{equation}
with $A_o\in\mathbb{R}^{n_o\times2}$ and $b_o\in\mathbb{R}^{n_o}$ being the halfspace constraints, $r_o\in\mathbb{R}^2$ being the position of the obstacle, $R_o\in\mathbb{R}^{2\times2}$ being the rotation matrix of the obstacle, and $n_o$ representing the number of edges for the obstacle. Then, if we use~\eqref{eq:toy_example_robot_constraint} and~\eqref{eq:toy_example_obstacle_constraint} as the constraints in~\eqref{eq:differentiable_collision}, we can solve for the CBF defined in~\eqref{eq:diffopt_cbf}. We use a proportional controller as the performance controller
\begingroup
\allowdisplaybreaks
\begin{subequations}
\begin{align}
    v &= K_v\sqrt{(p_{t, x} - p_x)^2 + (p_{t, y} - p_y)^2}\\
    \omega &= K_\omega\Big[\mathrm{atan2}(p_{t, y} - p_y, p_{t, x} - p_x) - \phi\Big]
\end{align}
\end{subequations}
\endgroup
with $K_v = 0.5$, $K_\omega = 2.0$, and the target position $(p_{t, x}, p_{t, y}) = (5.0, 3.0)$. The control is obtained by solving~\eqref{eq:cbf_opt_perf}. The generated motion is shown in Fig.~\ref{fig:toy_example}. Using our proposed CBF with a simple performance controller, the controller in~\eqref{eq:cbf_opt_perf} can generate fairly complex maneuvers to ensure safety. When testing the SDF-based CBF proposed in~\cite{SingletaryKA22}, the mobile robot gets stuck when getting close to the obstacle, and the CBFQP often fails to find a feasible solution. This is due to a reduced feasible set caused by the conservativeness of the approximated partial derivative of the SDF in the CBF constraint. On average, the SDF-CBF takes $9 \mu s$ to compute, while our proposed CBF takes $34 \mu s$. However, SDF-based CBFs failed to solve the task.

\subsection{7-DOF Robotic Arm}
\label{sec:robot_arm_exp}
For the FR3 experiments, we constructed two settings: two walls with a gap on each of them (referred as the \texttt{two-walls} task) and three blocks scattered in the workspace (referred as the \texttt{three-blocks} task). We encapsulate the links with ellipsoid-shaped bounding boxes for the \texttt{three-blocks} task and capsule-shaped bounding boxes for the \texttt{two-walls} task. In both settings, the end-effector is encapsulated with a sphere. Since the CoM of the two base links does not have relative translation with respect to the base, we do not need to encapsulate them. This gives us, in total, seven bounding boxes (two bounding boxes encapsulate the fifth link) and $7n_\mathcal{O}$ CBFs for each of the experiments. We use a resolved rate controller with joint centering to obtain the desired joint velocity subject to the CBF constraint. The resolved rate controller cost $\mathcal{J}_{r}$ and the joint-centering cost $\mathcal{J}_{c}$ are
\begingroup
\allowdisplaybreaks
\begin{subequations}
\begin{align}
    \mathcal{J}_{r} &= \big\|J(\theta)\dot{\theta}_\mathrm{des} - \Big[K_p(p_\mathrm{des} - p) + \dot{p}_\mathrm{des}\Big]\big\|_2^2\\
    \mathcal{J}_{c} &= \big\|\mathcal{N}(\theta)[\dot{\theta}_\mathrm{des} - K_p^\prime(\theta_\mathrm{nominal} - \theta)]\big\|_2^2
\end{align}
\end{subequations}
\endgroup
with $\theta \in\mathbb{R}^7$ representing the joint angles, $\mathcal{N}\in\mathbb{R}^{n_j \times n_j}$ being the nullspace projection matrix $\mathcal{N}(\theta) = \mathbf{I} - J^\dagger(\theta)J(\theta)$ and $J^\dagger\in\mathbb{R}^{n_j \times 6}$ representing the pseudo-inverse of $J$. Then, we solve for $u = \dot{\theta}_\mathrm{des}$ using
\begin{align}
    \min_{u}\ &\ \mathcal{J}_{r} + \epsilon\mathcal{J}_{c}\\
    \mathrm{subject\ to}\ &\ \frac{\partial\mathbf{H}}{\partial x}u \geq -\gamma\mathbf{H}(x)\nonumber
\end{align}
with $\epsilon\in\mathbb{R}_+$ being a weighting parameter and $G(x) = \mathbf{I}_{n_j \times n_j}$. The torque command to track $\dot{\theta}_\mathrm{des}$ is generated using a derivative controller with gravity compensation 
\begin{equation}
    \tau = K_d(\dot{\theta}_\mathrm{des} - \dot{\theta}) + \mathbf{G}_{g}(x) \in \mathbb{R}^{n_j}
\end{equation}
with $K_d \in \mathbb{R}^{n_j \times n_j}$ being a diagonal matrix and $\mathbf{G}_{g}:\mathbb{R}^n \rightarrow \mathbb{R}^n_j$ the generalized gravitational vector. For the \texttt{two-walls} task, the goal is to reach a target point within the gap of the wall in the back. For the \texttt{three-blocks} task, the robot to reach a point within the lowest block. In both cases, the desired end-effector position is set to be the target point $p_\mathrm{des} = p_\mathrm{target}$, and the desired end-effector velocity to be zero $\dot{p}_\mathrm{des} = \mathbf{0}_{3\times1}$. 

The results of our proposed method on these two tasks are shown in Fig.~\ref{fig:manipulator_exp_results}. For \texttt{two-walls}, the end-effector reaches the target position while ensuring safety. The end-effector reaches the closest point to the target that avoids collision for \texttt{three-blocks}. At each time step, our method computes $7n_\mathcal{O}$ CBFs, seven for each obstacle, and we show the change in the minimum value among the seven CBFs for each obstacle in Fig.~\ref{fig:manipulator_exp_results}. For \texttt{two-walls} and \texttt{three-blocks}, the average CBF computation time is 0.24ms and 0.20ms, respectively. For both tasks, in simulations and on the real robot, the velocity control is updated at 100Hz. 

We empirically show that the proposed CBF is continuously differentiable. The individual elements of $\partial\alpha/\mu$ for the real-world experiment on the \texttt{three-blocks} task are plotted in Fig.~\ref{fig:jacobians_of_cbf_paper}. We see that all partial derivatives are continuous, echoing our claim in Conjecture~\ref{conj:cbf_condition}. For the \texttt{two-walls} task, we see that when continuous differentiability does not hold globally, the approach is still viable in practice and safety can be achieved using our proposed CBF. This is because even when the CBF is not globally continuously differentiable, such discontinuities would intuitively be expected to be over a sparse set (e.g., a set of measure zero). Hence, in practice, the likelihood of encountering the exact poses for which the CBF is not continuously differentiable will be minimal. This is corroborated by the empirical observation that we did not observe any undesired behaviors, e.g., no unsafe behavior or large control commands. The simulation and experimental results can be found at \url{https://youtu.be/WhfFZT1oyJE}.
\section{Conclusion}
\label{sec:conclusion}

This paper presents a systematic and computationally simple approach to construct CBFs using differentiable optimization based collision detectors. We showed continuous differentiability for strongly convex scaling functions when the gradients and Hessians exist and are continuous. We conjecture that the continuous differentiability can be generalized to one object being strongly convex and the second being only convex. We experimentally demonstrated the efficacy of our approach on a mobile robot in simulation and a 7-DOF robot manipulator in both simulations and on the real robot. In the future, we plan to extend our approach to consider input and state constraints and apply them to multi-robot collaboration tasks.

\bibliographystyle{IEEEtran}
\bibliography{IEEEabrv, refs.bib}

\begin{thebibliography}{10}
\providecommand{\url}[1]{#1}
\csname url@samestyle\endcsname
\providecommand{\newblock}{\relax}
\providecommand{\bibinfo}[2]{#2}
\providecommand{\BIBentrySTDinterwordspacing}{\spaceskip=0pt\relax}
\providecommand{\BIBentryALTinterwordstretchfactor}{4}
\providecommand{\BIBentryALTinterwordspacing}{\spaceskip=\fontdimen2\font plus
\BIBentryALTinterwordstretchfactor\fontdimen3\font minus \fontdimen4\font\relax}
\providecommand{\BIBforeignlanguage}[2]{{%
\expandafter\ifx\csname l@#1\endcsname\relax
\typeout{** WARNING: IEEEtran.bst: No hyphenation pattern has been}%
\typeout{** loaded for the language `#1'. Using the pattern for}%
\typeout{** the default language instead.}%
\else
\language=\csname l@#1\endcsname
\fi
#2}}
\providecommand{\BIBdecl}{\relax}
\BIBdecl

\bibitem{DaiHKK23}
B.~Dai, H.~Huang, P.~Krishnamurthy, and F.~Khorrami, ``Data-efficient control barrier function refinement,'' in \emph{Proceedings of American Control Conference, San Diego, CA}, May 2023, pp. 3675--3680.

\bibitem{DaiKPK23}
B.~Dai, P.~Krishnamurthy, A.~Papanicolaou, and F.~Khorrami, ``State constrained stochastic optimal control for continuous and hybrid dynamical systems using {DFBSDE},'' \emph{Automatica}, vol. 155, p. 111146, 2023.

\bibitem{AmesCENST19}
A.~D. Ames, S.~Coogan, M.~Egerstedt, G.~Notomista, K.~Sreenath, and P.~Tabuada, ``Control barrier functions: Theory and applications,'' in \emph{Proceedings of European Control Conference, Naples, Italy}, June 2019, pp. 3420--3431.

\bibitem{DaiKPK21}
B.~Dai, P.~Krishnamurthy, A.~Papanicolaou, and F.~Khorrami, ``State constrained stochastic optimal control using {LSTM}s,'' in \emph{Proceedings of American Control Conference, New Orleans, LA}, May 2021, pp. 1294--1299.

\bibitem{NguyenHGAS16}
Q.~Nguyen, A.~Hereid, J.~W. Grizzle, A.~D. Ames, and K.~Sreenath, ``3d dynamic walking on stepping stones with control barrier functions,'' in \emph{Proceedings of {IEEE} Conference on Decision and Control, Las Vegas, NV}, December 2016, pp. 827--834.

\bibitem{RobeyHLZDTM20}
A.~Robey, H.~Hu, L.~Lindemann, H.~Zhang, D.~V. Dimarogonas, S.~Tu, and N.~Matni, ``Learning control barrier functions from expert demonstrations,'' in \emph{Proceedings of {IEEE} Conference on Decision and Control, Jeju Island, South Korea}, December 2020, pp. 3717--3724.

\bibitem{LiZNLLB21}
C.~Li, Z.~Zhang, A.~Nesrin, Q.~Liu, F.~Liu, and M.~Buss, ``Instantaneous local control barrier function: An online learning approach for collision avoidance,'' \emph{CoRR}, vol. abs/2106.05341, 2021.

\bibitem{DaiKK22}
B.~Dai, P.~Krishnamurthy, and F.~Khorrami, ``Learning a better control barrier function,'' in \emph{Proceedings of {IEEE} Conference on Decision and Control, Canc\'{u}n, Mexico}, December 2022, pp. 945--950.

\bibitem{SingletaryKA22}
A.~Singletary, S.~Kolathaya, and A.~D. Ames, ``Safety-critical kinematic control of robotic systems,'' \emph{{IEEE} Control Systems Letters}, vol.~6, pp. 139--144, 2022.

\bibitem{MurtazaAWH22}
M.~A. Murtaza, S.~Aguilera, M.~Waqas, and S.~Hutchinson, ``Safety compliant control for robotic manipulator with task and input constraints,'' \emph{{IEEE} Robotics \& Automation Letters}, vol.~7, no.~4, pp. 10\,659--10\,664, 2022.

\bibitem{ThirugnanamZS22}
A.~Thirugnanam, J.~Zeng, and K.~Sreenath, ``Safety-critical control and planning for obstacle avoidance between polytopes with control barrier functions,'' in \emph{Proceedings of {IEEE} International Conference on Robotics and Automation, Philadelphia, PA}, May 2022, pp. 286--292.

\bibitem{SingletaryGMSA22}
A.~Singletary, W.~Guffey, T.~G. Moln{\'{a}}r, R.~W. Sinnet, and A.~D. Ames, ``Safety-critical manipulation for collision-free food preparation,'' \emph{{IEEE} Robotics \& Automation Letters}, vol.~7, no.~4, pp. 10\,954--10\,961, 2022.

\bibitem{GilbertJK88}
E.~G. Gilbert, D.~W. Johnson, and S.~S. Keerthi, ``A fast procedure for computing the distance between complex objects in three-dimensional space,'' \emph{{IEEE} Journal of Robotics and Automation}, vol.~4, no.~2, pp. 193--203, 1988.

\bibitem{Bergen01}
G.~Van Den~Bergen, ``Proximity queries and penetration depth computation on 3d game objects,'' in \emph{Proceedings of Game Developers Conference, San Jose, CA}, vol. 170, 2001.

\bibitem{TracyHM22}
K.~Tracy, T.~A. Howell, and Z.~Manchester, ``Differentiable collision detection for a set of convex primitives,'' in \emph{Proceedings of {IEEE} International Conference on Robotics and Automation, London, United Kingdom}, May 2023, pp. 3663--3670.

\bibitem{AmosK17}
B.~Amos and J.~Z. Kolter, ``Opt{N}et: Differentiable optimization as a layer in neural networks,'' in \emph{Proceedings of International Conference on Machine Learning, Sydney, Australia}, vol.~70, August 2017, pp. 136--145.

\bibitem{AgrawalABBDK19}
A.~Agrawal, B.~Amos, S.~T. Barratt, S.~P. Boyd, S.~Diamond, and J.~Z. Kolter, ``Differentiable convex optimization layers,'' in \emph{Proceedings of Annual Conference on Neural Information Processing Systems, Vancouver, Canada}, December 2019, pp. 9558--9570.

\bibitem{Dini07}
U.~Dini, \emph{Lezioni di analisi infinitesimale}.\hskip 1em plus 0.5em minus 0.4em\relax Stabilimento Tipografico Successori Fratelli Nistri, Pisa, Italy, 1907, vol.~1.

\bibitem{Border85}
K.~C. Border, \emph{Fixed Point Theorems with Applications to Economics and Game Theory}.\hskip 1em plus 0.5em minus 0.4em\relax Cambridge Univ. Press, Cambridge, U.K., 1985.

\bibitem{SchulmanDHLABPPGA14}
J.~Schulman, Y.~Duan, J.~Ho, A.~X. Lee, I.~Awwal, H.~Bradlow, J.~Pan, S.~Patil, K.~Goldberg, and P.~Abbeel, ``Motion planning with sequential convex optimization and convex collision checking,'' \emph{The International Journal of Robotics Research}, vol.~33, no.~9, pp. 1251--1270, 2014.

\bibitem{GilbertJ85}
E.~G. Gilbert and D.~W. Johnson, ``Distance functions and their application to robot path planning in the presence of obstacles,'' \emph{{IEEE} Journal of Robotics and Automation}, vol.~1, no.~1, pp. 21--30, 1985.

\bibitem{MolnarCSUA22}
T.~G. Moln{\'{a}}r, R.~K. Cosner, A.~W. Singletary, W.~Ubellacker, and A.~D. Ames, ``Model-free safety-critical control for robotic systems,'' \emph{{IEEE} Robotics \& Automation Letters}, vol.~7, no.~2, pp. 944--951, 2022.

\end{thebibliography}
\end{document}